\documentclass{article}
\usepackage{fullpage}

\usepackage{natbib}

\usepackage{amsmath,amsthm,amssymb,bm}
\usepackage{hyperref,cleveref,graphicx,multirow,pifont}
\usepackage{subcaption}
\usepackage{xspace}

\newtheorem{theorem}{Theorem}
\newtheorem{lemma}[theorem]{Lemma}

\newtheorem{definition}[theorem]{Definition}
\newtheorem{remark}{Remark}[theorem]
\newtheorem{property}[theorem]{Property}

\DeclareMathOperator*{\argmin}{arg\,min}

\def\calX{\mathcal{X}}
\def\KL{\mathrm{KL}}
\def\dmu{\mathrm{d}\mu}
\def\bbR{\mathbb{R}}

\def\JS{\mathrm{JS}}
\def\dx{\mathrm{d}x}
\def\dy{\mathrm{d}y}
\def\dz{\mathrm{d}z}
\def\bbR{\mathbb{R}}
\def\KL{\mathrm{KL}}
\def\TV{\mathrm{TV}}
\def\kl{\mathrm{kl}}
\def\diag{\mathrm{diag}}
\def\calX{\mathcal{X}}
\def\calY{\mathcal{Y}}
\def\calC{\mathcal{C}}

\def\eqdef{{:=}}

\def\tr{{\mathrm{tr}}}
\def\st{{\ :\ }}
\def\inner#1#2{{\langle #1,#2\rangle}}

\graphicspath{{./Fig/}{../Fig/}}

\makeatletter
\DeclareRobustCommand\onedot{\futurelet\@let@token\bmv@onedotaux}
\def\bmv@onedotaux{\ifx\@let@token.\else.\null\fi\xspace}

\def\wrt{w.r.t\onedot}

\makeatother

\begin{document}

\title{Chain Rule Optimal Transport}
\author{Frank Nielsen\\
Sony Computer Science Laboratories Inc\\
Tokyo, Japan\\
\ \\
{\small ORCID: 0000-0001-5728-0726}\\
E-mail: {\tt Frank.Nielsen@acm.org}\\
\and
Ke Sun\\
CSIRO's Data61\\
Sydney, Australia\\
\ \\
{\small ORCID: 0000-0001-6263-7355}\\
E-mail: {\tt sunk@ieee.org}
}

\date{}
\maketitle

\begin{abstract}
We define a novel class of distances between statistical multivariate distributions by modeling an optimal transport problem on their marginals with respect to a ground distance defined on their conditionals.
These new distances are metrics whenever the ground distance between the marginals is a metric, generalize both the Wasserstein distances between discrete measures and a recently introduced metric distance between statistical mixtures, and provide an upper bound for jointly convex distances between statistical mixtures.
By entropic regularization of the optimal transport, we obtain a fast differentiable Sinkhorn-type distance.
We experimentally evaluate our new family of distances by quantifying the upper bounds of several jointly convex distances between statistical mixtures, and by proposing a novel efficient method to learn  Gaussian mixture models (GMMs) by simplifying  kernel density estimators with respect to our distance.
Our GMM learning technique experimentally improves significantly over the EM implementation of {\tt sklearn} on the {\tt MNIST} and {\tt Fashion MNIST} datasets.
\end{abstract}


\section{Introduction and motivation}

Calculating dissimilarities between statistical mixtures is a fundamental operation  met in statistics, machine learning,
signal processing, and information fusion~\citep{MixtureFusion-2010} among others.
Minimizing the information-theoretic Kullback-Leibler divergence (KLD also called relative entropy)
between parametric models yields practical learning machines.
However, the KLD or in general the Csisz\'ar's $f$-divergences between statistical mixtures~\citep{arxivlse-2016,distlse-2016} do not admit closed-form formula, and needs in practice to be approximated by costly Monte Carlo stochastic integration.
To tackle this computational tractability problem, two research directions have been considered in the literature:
\ding{192} propose some new distances between mixtures that yield closed-form formula~\citep{mixdist-2012,StatMinkowski-2019}
(e.g., the Cauchy-Schwarz divergence, the Jensen quadratic R\'enyi divergence, the statistical Minkowski distances).
\ding{193} lower and/or upper bound the $f$-divergences between mixtures~\citep{KLGMM-2012,distlse-2016}.
However, this direction is tricky when considering bounded divergences like the Total Variation (TV) distance or the Jensen-Shannon (JS) divergence that are upper bounded by $1$ and $\log 2$, respectively, or when considering high-dimensional mixtures.

When dealing with probability densities, two main classes of statistical distances have been widely studied in the literature:
\ding{192}
The invariant $f$-divergences of Information Geometry~(\citealt{IG-2016}; IG) characterized as the class of separable distances which are information monotone (i.e., satisfies the partition inequality~\cite{GenTsallisDiv-2019}),
and
\ding{193} The Optimal Transport (OT)/Wasserstein/EMD distance~\citep{Monge-1781,bookOT-2015} which can be computationally accelerated using entropy regularization~\citep{Sinkhorn-2013,OTMMD-2018} (i.e., the Sinkhorn divergence).

In general, computing closed-form formula for the OT between parametric distributions is difficult except in 1D~\citep{COT-2019}.
A closed-form formula is known for elliptical distributions~\citep{OTElliptical-1982} for the $2$-Wasserstein metric (including the multivariate Gaussian distributions), and the OT of multivariate continuous distributions can be calculated from the OT of their copulas~\citep{mvnOT-2018}.

The geometry related to these OT/IG distances are different.
For example, consider univariate location-scale families (or multivariate elliptical distributions):
For OT, the $2$-Wasserstein distance between any two members admit the {\em same} closed-form formula~\citep{OTElliptical-1982,Gelbrich-1990} (depending only on the mean and variance parameters, and not on the type of location-scale family).
The OT geometry of Gaussian distributions has {\em positive curvature}~\citep{OTgeometry-1996,WassersteinGaussian-2011}.
For any smooth $f$-divergence, the information-geometric manifold has {\em negative curvature}~(\citealt{komaki-2007}; hyperbolic geometry).

In this chapter, we first generalize the work of~\cite{OTWeight-2000} that proposed a novel family of statistical distances  between statistical mixtures  (that we term MCOTs, standing for {\em Mixture Component Optimal Transports})
by solving linear programs between mixture component weights  where the elementary distance between any two mixture components is prescribed.
Then we propose to learn  Gaussian mixture models (GMMs) by simplifying kernel density estimators (KDEs) using our distance.

We describe our main contributions as follows:
\begin{itemize}

\item We define the generic {\em Chain Rule Optimal Transport} (CROT) distance in Definition~\ref{def:CROT}, and prove that the CROT distance is a metric whenever the distance between conditional distributions is a metric in Theorem~\ref{thm:CROTmetric}. The CROT distance unifies and extends the Wasserstein distances and the MCOT distance~\cite{OTWeight-2000} between statistical mixtures.

\item We report a novel generic upper bound for statistical distances between marginal distributions~\citep{TVMM-2018} in \S\ref{sec:CROTmix} (Theorem~\ref{thm:ubjcd}) whenever the ground distance is jointly convex, and introduce its relaxed Sinkhorn distance (SCROT) for fast estimation.  Numerical experiments in \S\ref{sec:exp} highlight quantitatively the upper bound performance of the (S)CROT distances for bounding the total variation distance, the Wasserstein $W_p$ metric, and the R\'enyi $\alpha$-divergences.

\item We design a novel learning algorithm for GMMs by simplifying KDEs with respect to SCROT that yields in that case a closed-form formula (Eq.~\ref{eq:normalize}) in \S\ref{sec:GMM}, and demonstrate experimentally better results than the Expectation-Maximization (EM) implementation~\citep{EM-1977} in {\tt sklearn}~\cite{scikit} on MNIST~\citep{MNIST-1998} and Fashion MNIST~\citep{xiao2017} datasets.

\end{itemize}

\section{Chain Rule Optimal Transport}

Recall the basic {\em chain rule} factorization of a joint probability distribution:
\begin{equation*}
p(x,y) = p(y)\,p(x|y),
\end{equation*}
where probability $p(y)$ is   the {\em marginal probability},
and probability $p(x|y)$ is   the {\em conditional probability}.
Given $p(y)$ and $p(x|y)$ in certain families of simple probability distributions,
one can get a density model through marginalization:
\begin{equation*}
p(x) = \int p(x,y) \dy.
\end{equation*}
For example, for latent models like statistical mixtures or hidden Markov models~\citep{HMMDist-2005,UB-KL-HMM-2006},
$x$ plays the role of the {\em observed variable} while $y$ denotes the {\em hidden variable}~\citep{LatentModel-2013}
(unobserved so that inference has to tackle incomplete data, say, using the EM algorithm~\citep{EM-1977}.
Let $\calX=\{p(x)\}$ and $\calY=\{p(y)\}$ denote the manifolds of marginal probability densities;
let $\calC=\{p(x|y)\}$ denote the manifold of conditional probability density.
We state the generic definition of the {\em Chain Rule Optimal Transport} (CROT) distance between the distributions $p(x)$ and $q(x)$
(with $q(x) = \int q(y) q(x|y) \dy$) as follows:

\begin{definition}[CROT distance]\label{def:CROT}
Given two multivariate distributions $p(x,y)$ and $q(x,y)$,
we define the {\em Chain Rule Optimal Transport} as follows:
\begin{equation}
H_D(p, q)
\eqdef
\inf_{r\in\Gamma(p(y),q(z))}
E_{r(y,z)} \left[ D\bigg(p(x|y), q(x|z)\bigg) \right],
\end{equation}
where
$D(\cdot,\cdot)$ is a ground distance defined on conditional density manifold $\calC=\{p(x|y)\}$  (e.g., the Total Variation),
$\Gamma(p(y),q(z))$ is the set of all probability measures on $\calY^2$
satisfying the constraints $\int r(y,z) \dz  = p(y)$ and $\int r(y,z) \dy  = q(z)$,
and $E_{r(y,z)}$ denotes the expectation with respect to~$r(y,z)$.
\end{definition}

When the ground distance $D$ is clear from the context, we write $H(p, q)$ for a shortcut of $H_{D}(p, q)$.
A similar definition was introduced by~\cite{ruschendorf} termed ``Markov construction.''
In our work, the CROT
is defined with respect to a distance metric on the manifold $\calC$ of conditional densities (information-geometric distance)
rather than a section of the distance metric on the space of $(x,y)$.

A key property of CROT is stated as follows:

\begin{property}[Metric properties]\label{thm:CROTmetric}
If $D(\cdot,\cdot)$ is a metric on $\calC$,
then $H_D(p,q)$ is a metric on $\calX$ and a pseudometric on $\calY\times\calC$.
\end{property}
\noindent
The proof is given in \cref{thm:CROTmetric}.
Notice that $H_D$ is a metric on $\calX$ but only a pseudometric (satisfying 
non-negativity, symmetry, triangle inequality,
and $H_D(p,p)=0$, $\forall{p}\in\calY\times\calC$ instead of the law of indiscernibles of metrics)
on the product manifold $\calY\times\calC$. 

Since $\int r(y,z) \dy\dz=1$ and since $r(y,z)=p(y)q(z)$ is a feasible transport solution,
we get the following upper bounds:

\begin{property}[Upper bounds]\label{prop:maxUB}
\begin{align}
H_D(p, q) & \leq \int_y\int_z p(y)q(z)  D\bigg(p(x|y), q(x|z)\bigg) \mathrm{d}y\mathrm{d}z\nonumber\\
& \leq \max_{y,z} D\bigg(p(x|y), q(x|z)\bigg).
\end{align}
\end{property}

The CROT distances unify and generalize two distances met in the literature:

\begin{remark}[CROT generalizes Wasserstein/EMD]
In the case that $p(x|y)=\delta(x-y)$ (Dirac distributions), we recover the
Wasserstein distance~\citep{WassersteinGaussian-2011} between point sets
(or Earth Mover Distance, EMD;~\citealt{EMD-2000}), where $D(\cdot,\cdot)$ is the ground metric distance.
Note that point sets can be interpreted as discrete probability measures.
\end{remark}

The Wasserstein distance $W_p$ (for $p\geq 1$, with $W_1$ introduced by~\citealt{Wasserstein-1969})
follows from the Kantorovich's~\citeyearpar{Kantorovich-1942,Kantorovitch-1958} relaxation framework
of Monge's~\citeyearpar{Monge-1781} original optimal mass transport formulation.

\begin{remark}[CROT generalizes MCOT]
When both $p(y)$ and $q(z)$ are both (finite) categorical distributions,
we recover the distance formerly defined by~\cite{OTWeight-2000} that
we termed the MCOT distance.
\end{remark}

CROT is a nontrivial generalization of both the Wasserstein distance and the MCOT, because CROT
gives a {\em flexible} definition on the OT.
Given a joint distribution $p(x_1,\cdots,x_n)$, one can consider
a {\em family} of distances, depending on
how the random variables $x_1,\cdots,x_n$ split, and how the ground distances $D$ are selected.
For example, one can define $D$ to be CROT and we have a nested CROT distance.
In the simplest case, let
\begin{align}\label{eq:dcrot}
&D\left(p(x\vert{}y), q(x\vert{}z)\right)\nonumber\\
&= \inf_{r'\in\Gamma(p(x|y),q(x'|z))} E_{r'(x,x')} \Vert(x,y)-(x',z)\Vert_p,
\end{align}
then $H_D(p,q)$ becomes a ``two-stage optimal transport''
\begin{align}
H_D(p,q) = &\inf_{r\in\Gamma(p(y),q(z))} E_{r(y,z)}\inf_{r'\in\Gamma(p(x|y),q(x'|z))}\nonumber\\
&E_{r'(x,x')} \Vert(x,y)-(x',z)\Vert_p,
\end{align}
We have the following fundamental monotonicity:
\begin{theorem}
If $D(\cdot,\cdot)$ is given by \cref{eq:dcrot}, then we have:
\begin{equation}
H_D(p,q) \ge \inf_{r\in\Gamma(p(x,y),q(x',z))} E_{r}
\Vert(x,y)-(x',z)\Vert_p.\nonumber
\end{equation}
\end{theorem}
The above theorem is true if $\Vert(x,y)-(x',z)\Vert_p$ in \cref{eq:dcrot} and the RHS is replaced by any other metric distance.
Therefore, through the chain rule factorization of a joint distribution, CROT can give a potentially simpler expression of optimal transport, and its hierarchical structure allows one to use 1D OT problems~\citep{SlicedRadonWass-2015,SCDF-2019} which enjoys a closed-form solution~\citep{COT-2019} based on the inverse of the CDFs of the univariate densities:
$$
H_D(X,Y)= \left(\int_0^1 c_D(F_X^{-1}(u)-F_Y^{-1}(u)) \mathrm{d}u \right),
$$
where $F_X$ and $F_Y$ are the cumulative distribution functions (CDFs) of $X$ and $Y$, respectively, and $D(x,y):=c_D(x-y)$ for a convex and continuous function $C_D$.
Observe that the CROT distance is larger than the optimal transport distance.

Interestingly, the CROT distance provides an upper bound on the marginal distance $D(p(x),\,q(x))$
provided the base distance $D$ is {\em jointly convex}~\citep{ConvexBregman-2001,matBDjointcvx-2015}.

\begin{definition}[Jointly convex distance]
A distance $D(\cdot:\cdot)$ on a statistical manifold $\mathcal{M}$ is jointly convex if and only if
$$
D((1-\alpha)p_1+\alpha p_2 : (1-\alpha)q_1+\alpha q_2) \leq (1-\alpha)D(p_1:p_2)+\alpha D(p_2:q_2),\quad
\forall\alpha\in[0,1],\;p_1, p_2 \in\mathcal{M}.
$$
We write the above inequality more compactly as
$$
D((p_1p_2)_\alpha : (q_1q_2)_\alpha) \leq (D(p_1:p_2)D(p_2:q_2))_\alpha,\quad \forall\alpha\in [0,1],
$$
where $(ab)_\alpha\eqdef (1-\alpha)a+\alpha b$.
\end{definition}

\begin{theorem}[Upper Bound on Jointly Convex Distance, UBJCD]\label{thm:ubjcd}
Given a pair of joint distributions $p(x,y)$ and $q(x,y)$,
if $D(\cdot,\cdot)$ is jointly convex,
then  
$D(p(x),\,q(x)) \leq H_D(p,q)$.
\end{theorem}
Notice that $H_D(p,q)\not=H_D(q,p)$ for an asymmetric base distance $D$.

Let us give some examples of jointly convex distances:
\ding{192}
The $f$-divergences \citep{Osterreicher-2003} $I_f(p:q)=\int p(x)f(q(x)/p(x)) \mathrm{d}x$ (for a convex generator $f(u)$ satisfying $f(1)=0$ and strictly convex at $1$);
\ding{193}
The $p$-powered Wasserstein distances~\citep{WassersteinJC-2018} $W_p^p$;
\ding{194}
The R\'enyi divergences~\citep{Renyi-2014} for $\alpha\in [0,1]$;
\ding{195}
Bregman divergences~(\citep{VBVD-2007} ,\citealt{BorweinConvex-2010}; Exercises 2.3.29 and 2.3.30) provided that the generator $F$ satisfies
 $\nabla^2 F(y)+\nabla^3 F(y)(y-x)\succeq (\nabla^2F(x)\nabla^2F)^{-1}(y)$ where $\succeq$ denotes the L\"owner ordering of positive-definite matrices.
\ding{196} A generalized divergence related to Tsallis divergence~\cite{GenTsallisDiv-2019}.

A jointly convex function is separately convex but the converse is false.
However, a separately convex bivariate function that is positively homogeneous of degree one is jointly convex (but this result does not hold in higher dimensions;~\citealt{perspectivefunc-2008})
Conversely, CROT yields a lower bound for jointly concave distances (e.g., fidelity in quantum computing;~\citealt{QCQI-2002}).

\section{SCROT: Fast Sinkhorn CROT\label{sec:CROTmix}}

Consider two finite statistical mixtures $m_1(x)=\sum_{i=1}^{k_1} \alpha_i p_i(x)$ and
$m_2(x)=\sum_{i=1}^{k_2} \beta_i q_i(x)$, not necessarily homogeneous nor of the same type.
Let $[k]\eqdef \{1,\ldots, k\}$.
The MCOT distance proposed by~\cite{OTWeight-2000}
amounts to solve a {\em Linear Program} (LP) problem.

By defining $U(\alpha,\beta)$ to be set of non-negative matrices $W=[w_{ij}]$ with $\sum_{l=1}^{k_2} w_{il}=\alpha_i$
and $\sum_{l=1}^{k_1} w_{lj}=\beta_j$ (transport polytope;~\citealt{Sinkhorn-2013}), we get the equivalent compact definition of MCOT (that is a special case of CROT):
\begin{equation}
H_D(m_1,m_2) = \min_{W\in U(\alpha,\beta)} \sum_{i=1}^{k_1}\sum_{j=1}^{k_2} w_{ij} D(p_i,q_j).
\end{equation}
In general, the LP problem (with $k_1\times k_2$ variables and inequalities, $k_1+k_2$ equalities whom $k_1+k_2-1$ are independent) delivers an optimal soft assignment
of mixture components with exactly $k_1+k_2-1$ nonzero coefficients\footnote{A LP in $d$-dimensions has its solution located at a vertex of a polytope, described by the intersection of $d+1$ hyperplanes (linear constraints).} in matrix $W=[w_{ij}]$.
The complexity of linear programming~\citep{LP-2018} in $n$ variables with $b$ bits using Karmarkar's interior point methods is polynomial, in $O(n^{\frac{7}{2}}b^2)$.

Observe that we necessarily have:
$
\max_{j\in[k_2]} w_{ij} \geq \frac{\alpha_i}{k_2},
$
and similarly that:
$
\max_{i\in[k_1]} w_{ij} \geq \frac{\beta_j}{k_1}.
$
Note that $H(m,m)=0$ since $w_{ij}=D_{ij}$ where  $D_{ij}$ denotes the Kr\"onecker symbol:
$D_{ij}=1$ iff $i=j$, and $0$ otherwise.
We can interpret MCOT as a {\em Discrete Optimal Transport} (DOT) between (non-embedded) histograms.
When $k_1=k_2=d$, the transport polytope is the polyhedral set of non-negative $d\times d$ matrices:
$$
U(\alpha,\beta)=\{P \in \bbR_+^{d\times d} \st P 1_d=\alpha, P^\top 1_d=\beta\},
$$
and
$$
H_D(m_1:m_2) = \min_{P\in U(\alpha,\beta)} \inner{P}{W},
$$
where $\inner{A}{B}=\tr(A^\top B)$ is the Fr\"obenius inner product of matrices, and $\tr(A)$ the matrix trace.
This OT can be calculated using the network simplex in $O(d^3\log d)$ time.
Cuturi~\citeyearpar{Sinkhorn-2013} showed how to relax the objective function in order to get fast calculation using the Sinkhorn divergence:
\begin{equation}\label{eq:sd}
S_D(m_1:m_2) = \min_{P\in U_\lambda(\alpha,\beta)} \inner{P}{W},
\end{equation}
where 
$$
U_\lambda(\alpha,\beta)\eqdef \{P\in U(\alpha,\beta) \st \KL(P:\alpha\beta^\top) \leq \lambda \}.
$$
The KLD between two $k\times k$ matrices $M=[m_{i,j}]$ and $M'=[m_{i,j}']$ is defined by
$$
\KL(M:M')\eqdef\sum_{i,j}  m_{i,j}\log\frac{m_{i,j}}{m_{i,j}'},
$$
with the convention that $0\log \frac{0}{0}=0$.
The Sinkhorn divergence is calculated using the equivalent dual Sinkhorn divergence by using matrix scaling algorithms (e.g., the Sinkhorn-Knopp algorithm).
Because the minimization is performed on $U_\lambda(\alpha,\beta)\subset U(\alpha,\beta)$, we have
$$
H_D(m_1,m_2)   \leq S_D(m_1,m_2).
$$

Notice that the smooth (dual) Sinkhorn divergence has also been shown experimentally  to improve over
the EMD in applications (MNIST classification;~\citealt{Sinkhorn-2013}).

\subsection{CROT upper bounds on distance between statistical  mixtures\label{sec:ubcrot}}

First, let us report the basic upper bounds for MCOT mentioned earlier in Property~\ref{prop:maxUB}.
The objective function is upper bounded by:
\begin{equation}
H(m_1,m_2) \leq \sum_{i=1}^{k_1}\sum_{j=1}^{k_2} \alpha_i\beta_j D(p_i,q_j) \leq \max_{i\in [k_1],j\in [k_2]}  D(p_i,q_j).
\end{equation}
Now, when the conditional density distance $D$ is {\em separate convex} (i.e., meaning convex in both arguments),
we get the following {\em{}Separate Convexity Upper Bound}:
\begin{equation}\label{sec:SCUB}
\text{(SCUB)}\quad D(m_1:m_2) \leq \sum_{i=1}^{k_1} \sum_{j=1}^{k_2} \alpha_i\beta_j D(p_i:q_j).
\end{equation}
For example, norm-induced distances or $f$-divergences~\citep{fdiv-2014} are separate convex distances.
For the particular case of the KLD, we have:
$
\KL(p:q) \eqdef \int p(x)\log \frac{p(x)}{q(x)} \mathrm{d}x,
$
and when $k_1=k_2$, we get the following upper bound
using the log-sum inequality~\citep{FastKL-2003,wmixture-2017}:

\begin{equation}\label{eq:logsumUB}
\KL(m_1:m_2) \leq  \KL(\alpha:\beta)+\sum_{i=1}^k \alpha_i \KL(p_i:q_i),
\end{equation}

Since this holds for {\em any} permutation of $\sigma$ of mixture components, we can tight this upper bound by minimizing over all permutations $\sigma$:
\begin{equation}
\KL(m_1:m_2) \leq  \min_\sigma \KL(\alpha:\sigma(\beta))+\sum_{i=1}^k \alpha_i \KL(p_i:\sigma(q_i)).
\end{equation}

The best permutation $\sigma$ can be computed using the Hungarian algorithm~\citep{KLGMM-1999,GMM-UBM-2000,KLGMM-2003,GMMHungarian-2005}
 in cubic time (with cost matrix $C=[c_{ij}]$, and  $c_{ij} = \kl(\alpha_i : \beta_j) + \alpha_i\KL(p_i:q_j)$ with $\kl(a:b)=a\log\frac{a}{b}$).

Now, let us further rewrite  
$$
m_1(x)=\sum_{i=1}^{k_1} \sum_{j=1}^{k_2} w_{i,j} p_i(x)
$$ with $\sum_{j=1}^{k_2} w_{i,j}=\alpha_i$,
and 
$$
m_2(x)=\sum_{i=1}^{k_1} \sum_{j=1}^{k_2} w_{i,j}' q_j(x)
$$ 
with $\sum_{i=1}^{k_1} w_{i,j}'=\beta_j$.
That is, we can interpret 
$$
m_1(x)=\sum_{i=1}^{k_1} \sum_{j=1}^{k_2} w_{i,j} p_{i,j}(x)
$$ 
and 
$$
m_2(x)=\sum_{i=1}^{k_1} \sum_{j=1}^{k_2} w_{i,j}' q_{i,j}(x)
$$ 
as mixtures of $k=k_1\times k_2$ (redundant) components $\{p_{i,j}(x)=p_i(x)\}$ and $\{q_{i,j}(x)=q_j(x)\}$, and apply the upper bound of Eq.~\ref{eq:logsumUB} for the ``best split'' of matching mixture components $\sum_{j=1}^{k_2} w_{i,j} p_i(x) \leftrightarrow \sum_{j=1}^{k_1} w_{j,i}' q_i(x)$:

\begin{eqnarray*}
\lefteqn{\KL(m_1:m_2) \leq O(m_1:m_2) \leq }\\
&& \sum_{i=1}^{k_1} \sum_{j=1}^{k_2} w_{i,j}\log \left(\frac{w_{i,j}}{w_{j,i}'}\right)  + H_\KL(m_1,m_2),
\end{eqnarray*}
where
\begin{align}\label{eq:OD}
O(m_1:m_2)= \min_{w\in U(\alpha,\beta)} & \sum_{i=1}^{k_1} \sum_{j=1}^{k_2} w_{i,j}\log\left( \frac{w_{i,j}}{w_{j,i}'}\right)\nonumber\\
+ & \sum_{i=1}^{k_1} \sum_{j=1}^{k_2} w_{ij} \KL(p_i:q_j).
\end{align}

Thus CROT allows to upper bound the KLD between mixtures.
The technique of rewriting mixtures as mixtures of $k=k_1\times k_2$ redundant components bears some resemblance with the variational upper bound on the KL divergence between mixtures proposed by~\cite{KLGMMVA-2007} that requires to iterate until convergence an update of the variational upper bound. See also~\cite{OTGMM-2019} for another recent work further pushing that research direction and discussing displacement interpolation and barycenter calculations for Gaussian Mixture Models (GMMs).
We note that this framework  also applies to or semi-parametric mixtures obtained from Kernel Density Estimators (KDEs;~\citealt{MixtKDE-2013}).

\section{Experiments\label{sec:exp}}
We   study experimentally the tightness of the CROT upper bound $H_D$ and SCROT upper bound $S_D$  on $D$ between GMMs for the total variation (\S\ref{exp:TV}), Wasserstein $W_p$ (\S\ref{exp:Wass}) and R\'enyi distances (\S\ref{exp:Renyi}). In~\S\ref{sec:GMM} we shall further demonstrate how to learn GMMs by minimizing the SCROT distance.

\subsection{Total Variation distance}\label{exp:TV}

Since $\TV$ is a metric $f$-divergence~\citep{fdivmetric-2007} bounded in $[0,1]$, so is MCOT.
The closed-form formula for the total variation between univariate Gaussian distributions
is reported by~\cite{GenChernoff-2014} using the erf function, and the other formula for
the total variation between Rayleigh distributions and Gamma distributions are given in~\cite{TVMM-2018}.

Figure~\ref{fig:tvcrot} illustrates the performances of the various lower/upper bounds on the total variation between mixtures of Gaussian, Gamma, and Rayleigh distributions with respect to the true value which is estimated using Monte Carlo samplings (consistent estimations).

The acronyms of the various bounds are as follows:
CELB: Combinatorial Envelope Lower Bound~(\citealt{distlse-2016}; applies only for 1D mixtures);
CEUB: Combinatorial Envelope Upper Bound~(\citealt{distlse-2016}; applies only for 1D mixtures);
CGQLB: Coarse-Grained Quantization Lower Bound~\citep{distlse-2016} for $1000$ bins (applies only for $f$-divergences that satisfy the information monotonicity property);
CROT: Chain Rule Optimal Transport $H_D$ (this paper);
Sinkhorn CROT: Entropy-regularized CROT~\citep{Sinkhorn-2013} $S_D\leq H_D$, with $\lambda=1$ and $\epsilon=10^{-8}$ (for convergence of the Sinkhorn-Knopp iterative matrix scaling algorithm).


Next, we consider the renown MNIST handwritten digit database~\citep{MNIST-1998} of 70,000 handwritten digit $28\times 28$ grey images
and the Fashion-MNIST images with exactly the same sample size and dimensions but different image contents~\citep{xiao2017}.
We first use PCA to reduce the original dimensionality $d=28\times 28=784$ to $D\in\{10,50\}$.
Then we extract two subsets of samples, and estimate respectively two GMMs composed
of 10 multivariate Gaussian distributions with a diagonal covariance matrix.
The GMMs are learned by the Expectation-Maximization (EM) algorithm implementation of {\sf scikit-learn}~\citep{scikit-2011}.
Notice that we did not use the labels in our estimation, and
therefore the mixture components do not necessarily correspond to different digits.

We approximate the TV between $D$-dimensional GMMs using Monte Carlo by performing stochastic integration of the following integrals:
\begin{align*}
&
\TV(p,q)
\;\eqdef\;\frac{1}{2} \int \vert p(x) - q(x) \vert \dx = \\
&
\frac{1}{2m} \sum_{x_i\sim{}p(x)}
\frac{1-\exp(r(x_i))}{1+\exp(r(x_i))}
+\frac{1}{2m} \sum_{y_i\sim{}q(x)}
\frac{1-\exp(r(y_i))}{1+\exp(r(y_i))},
\end{align*}
where $\{x_i\}_{i=1}^m$ and $\{y_i\}_{i=1}^m$
are i.i.d. samples drawn from $p(x)$ and $q(x)$, respectively,
and $r(x)=\vert \log{p}(x) - \log{q}(x) \vert$.
In our experiments, we set $m=0.5\times10^4$.

To compute the CROT, we use the EMD and Sinkhorn implementations provided by the
Python Optimal Transport, {\sf POT}, library~\citep{POT-2017}. For Sinkhorn, we set the
entropy regularization strength as follows:
Sinkhorn (1) means $\mathrm{median}(M)$ and
Sinkhorn (10) means $\mathrm{median}(M)/10$,
where
$M$ is the metric cost matrix. For example, to compute CROT-TV,
$M$ is the pairwise TV distance matrix from all components
in the first mixture model to all components in the second mixture.
The maximum number of Sinkhorn iterations is $1000$, with a stop threshold of $10^{-10}$.

To get some intuitions, see Figure~\ref{fig:tvmnist} for the cost matrix
and the corresponding optimal transport matrix, where the cost is
defined by TV distance, and the dataset is PCA-processed MNIST.
We see that the transportation scheme tries to assign higher weights
to small cost pairs (blue region in the cost matrix).

\begin{figure}
\includegraphics[width=\textwidth]{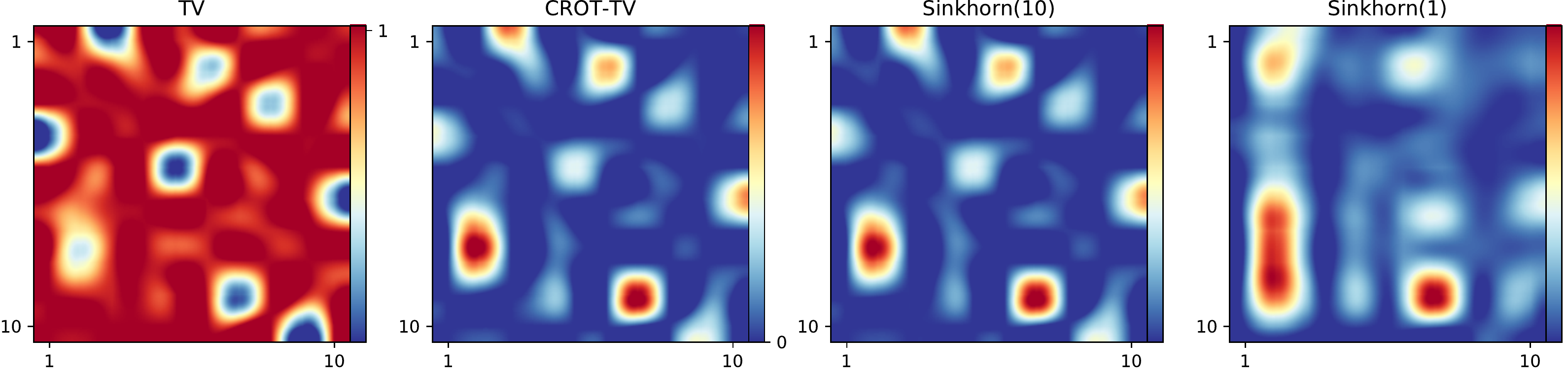}

\caption{TV distance between two $10$-component GMMs estimated on the MNIST dataset:
(1) shows the $10\times 10$ matrix TV distance between the first mixture components and the second mixture components
(red means large distance and blue means a small distance).
 (2-4) displays the $10\times 10$ optimal transport matrix $W$
 (red means larger weights, blue means smaller weights).
The optimal transport matrix is estimated by
EMD (2), the Sinkhorn algorithm with weak regularization (3) and the Sinkhorn with
strong regularization (4).}\label{fig:tvmnist}
\end{figure}

Figure 3(1) shows the 10x10 TV distance between mm1's components and mm2's components
red means large distance, blue means a small distance
Figure 3

\begin{table}[t]
\centering
\caption{TV distances between two GMMs with $10$ components each estimated on PCA-processed images.
$D$ is the dimensionality of the PCA. The two GMMs are estimated based on non-overlapping samples,
with the parameter $0<\tau\le1$ specifying the relative sample size used to estimated the GMMs.
For example, $\tau=1$ means each GMM is estimated on half of all available images.
Sinkhorn ($\lambda$) denotes the CROT distance estimated by
the Sinkhorn algorithm, where the regularization strength is proportional to $1/\lambda$.
For each configuration, the two GMMs are repeatedly estimated based on $100$ pairs of
random subsets of the full dataset, with the mean and standard deviation reported.
\label{tab:tvmnist}}

\noindent\scalebox{0.8}{
\noindent\vtop{\noindent
\begin{tabular}{ccc|ccccc}
\hline
Data & $D$ & $\tau$ & TV & CROT-TV & Sinkhorn ($10$) & Sinkhorn ($1$)\\
\hline
\multirow{4}{*}{MNIST}
& $10$ & $1$ & $0.16\pm0.08$ & $0.26\pm0.14$ & $0.27\pm0.14$ & $0.78\pm0.05$\\
& $10$ & $0.1$ & $0.29\pm0.05$ & $0.43\pm0.08$ & $0.44\pm0.08$ & $0.84\pm0.02$\\
& $50$ & $1$ & $0.35\pm0.08$ & $0.43\pm0.10$ & $0.44\pm0.10$ & $0.78\pm0.03$\\
& $50$ & $0.1$ & $0.54\pm0.04$ & $0.64\pm0.05$ & $0.67\pm0.06$ & $0.84\pm0.02$\\
\hline
& $10$ & $1$ & $0.19\pm0.09$ & $0.23\pm0.12$ & $0.24\pm0.12$ & $0.81\pm0.03$\\
Fashion
& $10$ & $0.1$ & $0.33\pm0.07$ & $0.40\pm0.09$ & $0.40\pm0.09$ & $0.86\pm0.02$\\
MNIST
& $50$ & $1$ & $0.44\pm0.11$ & $0.48\pm0.12$ & $0.50\pm0.13$ & $0.88\pm0.03$\\
& $50$ & $0.1$ & $0.60\pm0.07$ & $0.64\pm0.08$ & $0.67\pm0.09$ & $0.92\pm0.02$\\
\hline
\end{tabular}
}}
\end{table}

\def\ttt{1.0}
\begin{figure}
\centering
\includegraphics[width=\ttt\columnwidth]{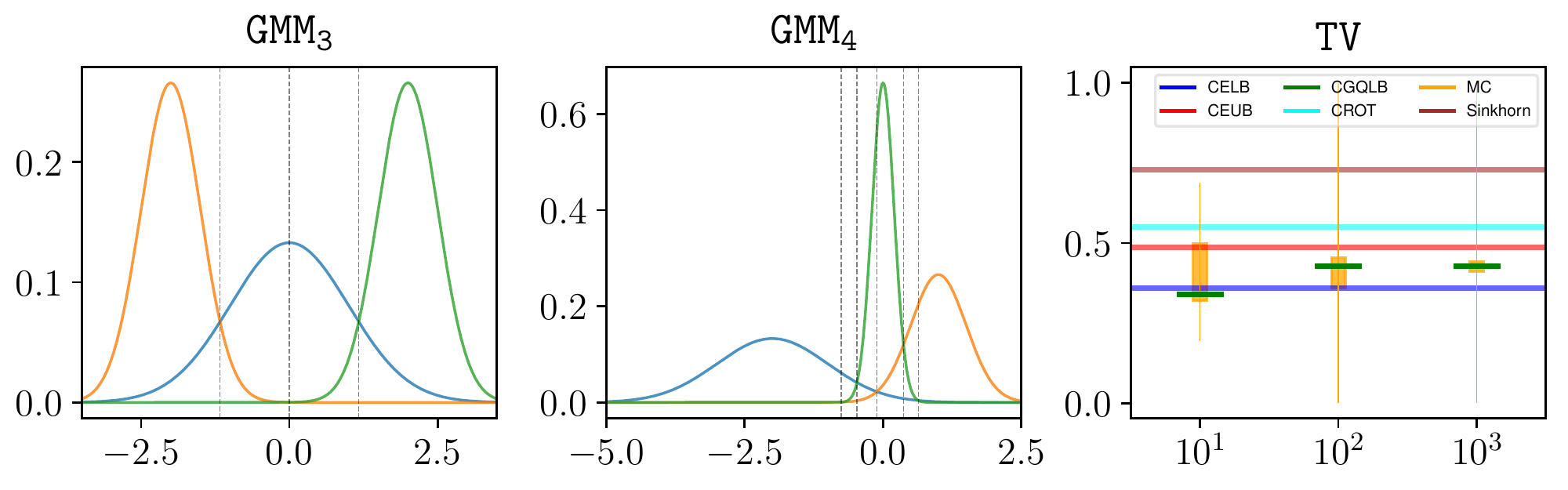}\\
\includegraphics[width=\ttt\columnwidth]{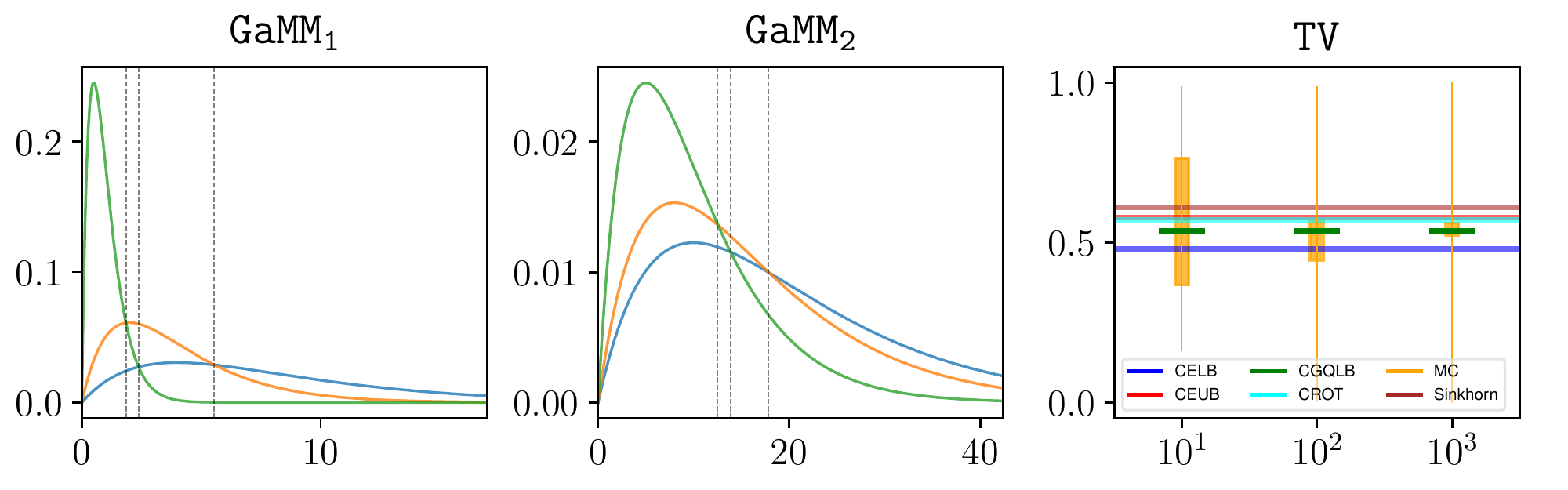}\\
\includegraphics[width=\ttt\columnwidth]{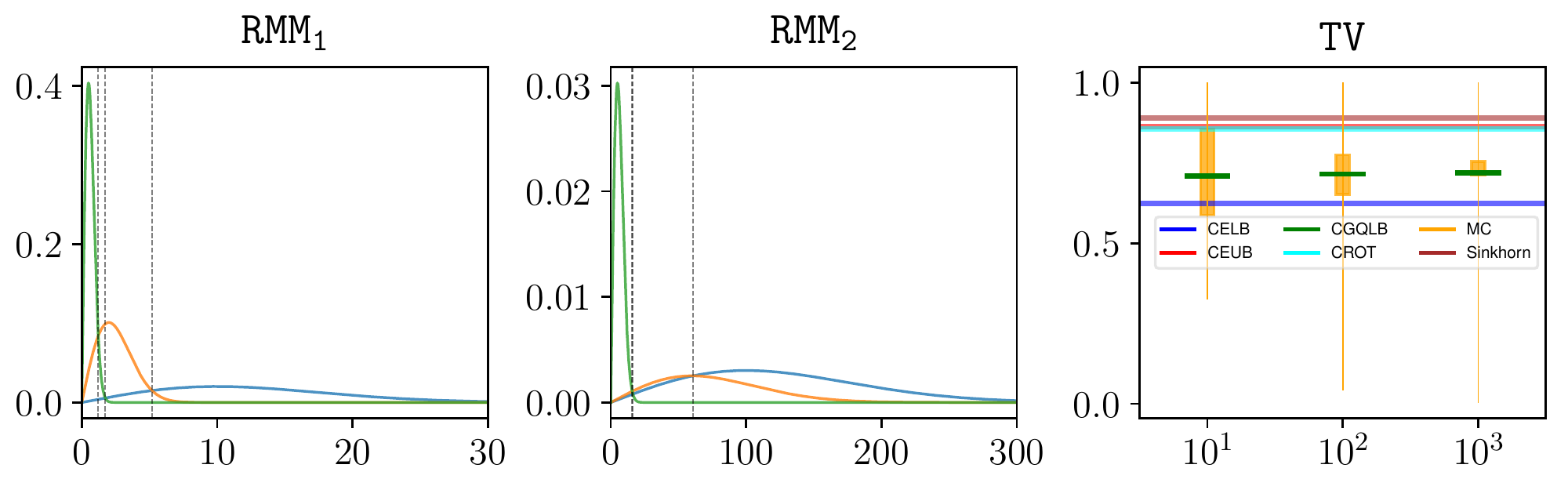}\\
\caption{Performance of the CROT distance and the Sinkhorn CROT distance for upper bounding the total variation distance between mixtures of (1) Gaussian, (2) Gamma, and (3) Rayleigh distributions.}%
\label{fig:tvcrot}
\end{figure}

Our experiments yield the following observations:
As the sample size $\tau$ decreases, the TV distances between GMMs
turn larger because the GMMs are
pulled towards the two different empirical distributions.
As the dimension $D$ increases, TV increases because in a high dimensional space the
GMM components are less likely to overlap.
We check that CROT-TV is an upper bound of TV.
We verify that Sinkhorn divergences are upper bounds of CROT.
These observations are consistent across two data sets.
The distances of Fashion-MNIST are in general larger than the corresponding distances in MNIST,
which can be intuitively explained by that the ``data manifold'' of Fashion-MNIST
has a more complicated structure than MNIST.

\begin{table}
\centering
\caption{$W_2$ distances between two 10-component GMMs estimated on PCA-processed images.\label{tab:w2mnist}}
\scalebox{0.8}{
\centering
\begin{tabular}{ccc|ccccc}
\hline
Data & $D$ & $\tau$ & $\mathrm{UB}(W_2)$ & $\mathrm{LB}(W_2)$ & $\sqrt{\text{CROT-}W_2^2}$ & Sinkhorn (10) & Sinkhorn (1)\\
\hline
\multirow{4}{*}{MNIST}
& $10$ & $1$ & $1.91\pm0.02$ & $0.03\pm0.00$ & $0.84\pm0.57$ & $0.88\pm0.58$ & $7.13\pm0.11$\\
& $10$ & $0.1$ & $1.93\pm0.02$ & $0.09\pm0.02$ & $1.48\pm0.38$ & $1.54\pm0.39$ & $7.29\pm0.11$\\
& $50$ & $1$ & $7.51\pm0.03$ & $0.07\pm0.01$ & $2.17\pm0.93$ & $2.39\pm0.97$ & $12.02\pm0.15$\\
& $50$ & $0.1$ & $7.53\pm0.04$ & $0.21\pm0.02$ & $4.04\pm0.86$ & $4.33\pm0.91$ & $12.69\pm0.22$\\
\hline
& $10$ & $1$ & $1.71\pm0.05$ & $0.03\pm0.01$ & $1.19\pm0.62$ & $1.24\pm0.63$ & $10.36\pm0.08$\\
Fashion
& $10$ & $0.1$ & $1.74\pm0.05$ & $0.10\pm0.02$ & $1.61\pm0.63$ & $1.68\pm0.64$ & $10.43\pm0.15$\\
MNIST
& $50$ & $1$ & $7.47\pm0.04$ & $0.07\pm0.01$ & $3.12\pm1.01$ & $3.21\pm1.02$ & $15.31\pm0.20$\\
& $50$ & $0.1$ & $7.50\pm0.04$ & $0.22\pm0.02$ & $4.32\pm1.02$ & $4.45\pm1.05$ & $15.99\pm0.29$\\
\hline
\end{tabular}
}
\end{table}

\subsection{Wasserstein $W_p$ CROT on GMMs}\label{exp:Wass}

The $p$-th power of the $L_p$-Wasserstein distance, $W_p^p$, is jointly convex for $p\geq 1$  (see Eq. 20, p. 6,~\citealt{WassersteinJC-2018}).
Thus we can apply the CROT distance between two GMMs $m_1$ and $m_2$ to get the following upper bound:
$W_p(m_1,m_2) \leq  H_{W_p^p}^{\frac{1}{p}}(m_1,m_2)$, $\alpha\geq1$.
We also have $W_p\leq W_q$ for $1\leq p\leq q<\infty$.

The OT distance $W_2$ between Gaussian measures~\citep{OTElliptical-1982,WassersteinGaussian-2011} is available in closed-form:
\begin{eqnarray*}
\lefteqn{W_2(N(\mu_1,\Sigma_1),N(\mu_1,\Sigma_1)) = }\\
&
\sqrt{
\|\mu_1-\mu_2\|^2 + \tr(\Sigma_1+\Sigma_2-2(\Sigma_1^{\frac{1}{2}}\Sigma_2\Sigma_1^{\frac{1}{2}})^{\frac{1}{2}})
}.
\end{eqnarray*}

This  $H_{W_p^p}^{\frac{1}{p}}$ CROT distance generalizes~\cite{OTGMM-2019}
who considered the $W_2$ distance between GMMs using discrete OT.
They proved that $H_{W_2}(m_1,m_2)$ is a metric, and $W_2(m_1,m_2)\leq \sqrt{H_{W_2^2}(m_1,m_2)}$.
These results generalize to mixture of elliptical distributions~\citep{OTElliptical-1982}.
However, we do not know a closed-form formula for $W_p$ between Gaussian measures when $p\not =2$.

Given two high-dimensional mixture models $m_1$ and $m_2$,
we draw respectively $n$ i.i.d. samples from $m_1$ and $m_2$, so that
$m_1(x) \approx \frac{1}{n} \sum_{i=1}^n D(x_i)$ and
$m_2(x) \approx \frac{1}{n} \sum_{j=1}^n D(y_j)$.
Then, we have
\begin{eqnarray}\label{eq:wpub}
W_p(m_1,m_2)
&\approx&
W_p\left( \frac{1}{n} \sum_{i=1}^n D(x_i), \frac{1}{n} \sum_{j=1}^n D(y_j) \right)\\
&\le&
H^{1/p}_{W_{p}^p}
\left( \frac{1}{n} \sum_{i=1}^n D(x_i), \frac{1}{n} \sum_{j=1}^n D(y_j) \right).\nonumber
\end{eqnarray}
Note that $W_p\left(D(x_i),D(x_j)\right)=\Vert{}x_i-x_j\Vert_2$
and therefore the RHS of \ref{eq:wpub} can be evaluated.
We use $\mathrm{UB}(W_2)$ to denote this empirical upper bound that will hold if $n\to\infty$.
In our experiments $n=10^3$.

See Table~\ref{tab:w2mnist} for the $W_2$ distances evaluated on the two
investigated data sets.
The column $\mathrm{LB}(W_2)$ is a lower bound based on the first
and second moments of the mixture models~\citep{Gelbrich-1990}.
We can clearly see that $\sqrt{H_{W_2^2}}$ provides a tighter upper bound
than $\mathrm{UB}(W_2)$. To compute $\mathrm{UB}(W_2)$ one need to draw
a potentially large number of random samples to make the approximation
in \ref{eq:wpub}, and the computation of the EMD is costly.
Therefore one should use $\sqrt{H_{W_2^2}}$ for its better and more efficient approximation.

\begin{table}
\centering
\caption{R\'enyi divergences between two $10$-component GMMs estimated on PCA-processed images.\label{tab:renyimnist}}
\noindent
\scalebox{0.8}{
\begin{tabular}{ccc|ccccc}
\hline
Data & $D$ & $\tau$ & $R_{\alpha}$ & $\text{CROT-}R_{\alpha}$ & Sinkhorn (10) & Sinkhorn (1)\\
\hline
& $10$ & $1$ & $0.01\pm0.01$ & $0.07\pm0.05$ & $0.08\pm0.05$ & $0.80\pm0.02$\\
MNIST
& $10$ & $0.1$ & $0.03\pm0.02$ & $0.15\pm0.04$ & $0.16\pm0.04$ & $0.84\pm0.04$\\
$R_{0.1}$
& $50$ & $1$ & $0.09\pm0.06$ & $0.25\pm0.09$ & $0.29\pm0.10$ & $1.40\pm0.07$\\
& $50$ & $0.1$ & $0.18\pm0.09$ & $0.42\pm0.09$ & $0.46\pm0.10$ & $1.43\pm0.09$\\
\hline
& $10$ & $1$ & $0.04\pm0.03$ & $0.11\pm0.06$ & $0.12\pm0.06$ & $1.59\pm0.05$\\
Fashion
& $10$ & $0.1$ & $0.06\pm0.03$ & $0.18\pm0.07$ & $0.19\pm0.07$ & $1.65\pm0.07$\\
MNIST
& $50$ & $1$ & $0.12\pm0.08$ & $0.30\pm0.11$ & $0.32\pm0.11$ & $2.37\pm0.08$\\
$R_{0.1}$
& $50$ & $0.1$ & $0.20\pm0.11$ & $0.45\pm0.10$ & $0.47\pm0.10$ & $2.41\pm0.10$\\
\hline
\hline
& $10$ & $1$ & $0.06\pm0.05$ & $0.34\pm0.23$ & $0.37\pm0.22$ & $4.09\pm0.12$\\
MNIST
& $10$ & $0.1$ & $0.17\pm0.05$ & $0.67\pm0.18$ & $0.72\pm0.18$ & $4.22\pm0.10$\\
$R_{0.5}$
& $50$ & $1$ & $0.31\pm0.13$ & $1.07\pm0.41$ & $1.28\pm0.43$ & $6.73\pm0.31$\\
& $50$ & $0.1$ & $0.69\pm0.14$ & $1.92\pm0.40$ & $2.16\pm0.42$ & $7.01\pm0.33$\\
\hline
& $10$ & $1$ & $0.17\pm0.12$ & $0.52\pm0.29$ & $0.55\pm0.29$ & $7.54\pm0.14$\\
Fashion
& $10$ & $0.1$ & $0.28\pm0.13$ & $0.87\pm0.28$ & $0.92\pm0.29$ & $7.79\pm0.23$\\
MNIST
& $50$ & $1$ & $0.54\pm0.24$ & $1.45\pm0.48$ & $1.55\pm0.48$ & $10.53\pm0.26$\\
$R_{0.5}$
& $50$ & $0.1$ & $0.89\pm0.21$ & $2.16\pm0.39$ & $2.27\pm0.40$ & $10.79\pm0.38$\\
\hline
\hline
& $10$ & $1$ & $0.14\pm0.09$ & $0.76\pm0.42$ & $0.80\pm0.42$ & $7.18\pm0.19$\\
MNIST
& $10$ & $0.1$ & $0.31\pm0.09$ & $1.35\pm0.37$ & $1.42\pm0.37$ & $7.53\pm0.35$\\
$R_{0.9}$
& $50$ & $1$ & $0.61\pm0.32$ & $1.90\pm0.82$ & $2.25\pm0.85$ & $12.46\pm0.66$\\
& $50$ & $0.1$ & $1.33\pm0.30$ & $3.51\pm0.80$ & $3.90\pm0.82$ & $12.96\pm0.86$\\
\hline
& $10$ & $1$ & $0.32\pm0.23$ & $1.07\pm0.60$ & $1.12\pm0.61$ & $14.25\pm0.38$\\
Fashion
& $10$ & $0.1$ & $0.50\pm0.26$ & $1.69\pm0.66$ & $1.77\pm0.67$ & $14.74\pm0.54$\\
MNIST
& $50$ & $1$ & $1.07\pm0.43$ & $2.76\pm0.96$ & $2.93\pm0.97$ & $21.41\pm0.78$\\
$R_{0.9}$
& $50$ & $0.1$ & $1.76\pm0.45$ & $4.18\pm1.06$ & $4.40\pm1.09$ & $22.16\pm1.02$\\
\hline
\end{tabular}
}
\end{table}

\subsection{R\'enyi CROT between GMMs}\label{exp:Renyi}
We investigate R\'enyi $\alpha$-divergence~\citep{RT-2011,SM-2011} defined by
$R_\alpha(p:q) = \frac{1}{1-\alpha}\log\int p(x)^\alpha q(x)^{1-\alpha} \dx$,
which encompasses KLD at the limit $\alpha\to1$.
Notice that for multivariate Gaussian densities $p$ and $q$,
$R_\alpha(p:q)$ can be undefined for $\alpha>1$ as the integral may diverge.
In this case the CROT-$R_\alpha$ divergence is undefined.
Table~\ref{tab:renyimnist} shows $R_\alpha$ for $\alpha\in\{0.1,0.5,0.9\}$
and the corresponding CROT estimated on MNIST and Fashion-MNIST datasets.
The observation is consistent with the other distance metrics.

\section{Learning GMMs with SCROT.KL}\label{sec:GMM}

This section performs an experimental study to learn mixture models using
SCROT. The observed data samples $\{\bm{x}_i\}_{i=1}^n$ is described by a kernel density estimator (KDE)
\begin{equation}
p(\bm{x})
= \frac{1}{n}\sum_{i=1}^n p_i(\bm{x})
= \frac{1}{n}\sum_{i=1}^n N( \bm{x}_i, \epsilon I ),
\end{equation}
where $\epsilon>0$ is a hyper parameter. We aim to learn a Gaussian mixture model
\begin{equation}
q(\bm{x})
= \sum_{i=1}^m \alpha_i q_i(\bm{x})
= \sum_{i=1}^m \alpha_i N(\bm{\mu}_i, \diag(\bm\sigma_i) ),
\end{equation}
where $\alpha_i\ge0$ ($\sum_{i=1}^m\alpha_i=1$) is the mixture weight
of $i$'s component, and diagonal covariance matrices are assumed
to reduce the number of free parameters.
Minimizing $\KL(p\,:\,q)$ gives the maximum likelihood
estimation~\citep{IG-2016}. However, the KLD between Gaussian mixture models
is known to be not having analytical form~\citep{distlse-2016}.
Therefore one has to rely on
variational bounds or the re-parametrization trick~\citep{vae} to
bound/approximate $\KL(p\,:\,q)$.
The CROT gives an alternative approach to minimize the KLD by simplifying a KDE~\citep{MixtKDE-2013}.
By \cref{thm:ubjcd}, we have $H_{\KL}(p\,:\,q) \ge \KL(p\,:\,q)$.
Therefore we minimize the upper bound $H_{\KL}(p\,:\,q)$ instead, which
can be computed conveniently as the KLD between Gaussian distributions
is in closed form. Moreover, because the mixture weights are free
parameters, the entropy-regularized optimal transport problem
is simplified into
\begin{align*}
\min_{\bm{W}} \sum_{i=1}^n \sum_{j=1}^m \bigg[ &
w_{ij} \KL\left(p_i,\,q_j\right) + \frac{1}{\lambda}w_{ij}\log{w}_{ij} \bigg],\\
\text{s.t.}\hspace{2em}& w_{ij}\ge0,\hspace{2em}\forall{i},\forall{j}\\
&\sum_{j=1}^m w_{ij} = \frac{1}{n},
\end{align*}
where $\lambda>0$ is a regularization strength parameter (same as the Sinkhorn algorithm).
By a similar analysis~\citep{Sinkhorn-2013}, the optimal weights $w_{ij}^\star$ must
satisfy
\begin{equation}\label{eq:normalize}
w_{ij}^{\star} = \frac{1}{n} \frac{\exp(-\lambda \KL(p_i,q_j))}{\sum_{j=1}^m\exp(-\lambda \KL(p_i,q_j))},
\end{equation}
We therefore minimize $ \sum_{i=1}^{n'} \sum_{j=1}^m w_{ij}^{\star} \KL(p_i,\,q_j)$
based on gradient descent on mini-batches of $n'$ samples.
We set empirically the hyper-parameter
$m=10$ (number of components),
$\lambda=0.005$ (Sinkhorn regularization parameter)
and $\epsilon=10^{-6}$ (KDE bandwidth).
Fine tuning them can potentially yields better results.
We use the training dataset to learn the $q$ distribution (GMM) and estimate the testing error
based on its distance with $\hat{p}$, a KDE \wrt the testing datasets.

\Cref{fig:mnist} shows the learning curves when estimating a 10-component-GMM
on MNIST (left) and Fashion MNIST (right).
One can observe that SCROT.KL is indeed an upper bound of KL.
Minimizing SCROT.KL can effectively learn a mixture model
on these two datasets. The resulting model achieves better testing error
as compared to \texttt{sklearn}'s EM algorithm~\citep{scikit}.  This is because we use
KDE as the data distribution, which better describes the data as compared
to the empirical distribution.
Comparatively, the KLD is larger on the Fashion MNIST dataset,
where the data distribution is more complicated and cannot be well
described by the GMM.
EM takes 2 minutes. SCROT is implemented in Tensorflow~\cite{TF-2016} using gradient descent (Adam), and takes
around 20 minutes for $100$ epochs on an Intel \texttt{i5-7300U} CPU.

In order to efficiently estimate the KLD (corresponding to ``KL'' and ``KL(EM)'' in the figure),
we use the information-theoretical bound
$H(X,Y)\le{}H(X)+H(Y)$, where $H$ denotes Shannon's entropy. Therefore
$\KL(p:q)=-H(p) - \int p(\bm{x}) \log{q}(\bm{x})\dx
\ge -H(U) - H(p_i) - \int p(\bm{x}) \log{q}(\bm{x})\dx$,
where $U=(1/n,\cdots,1/n)$ is the uniform distribution,
and the integral $\int p(\bm{x}) \log{q}(\bm{x})\dx$ is estimated by Monte-Carlo sampling.

\begin{figure}
\centering
\includegraphics[width=.8\textwidth]{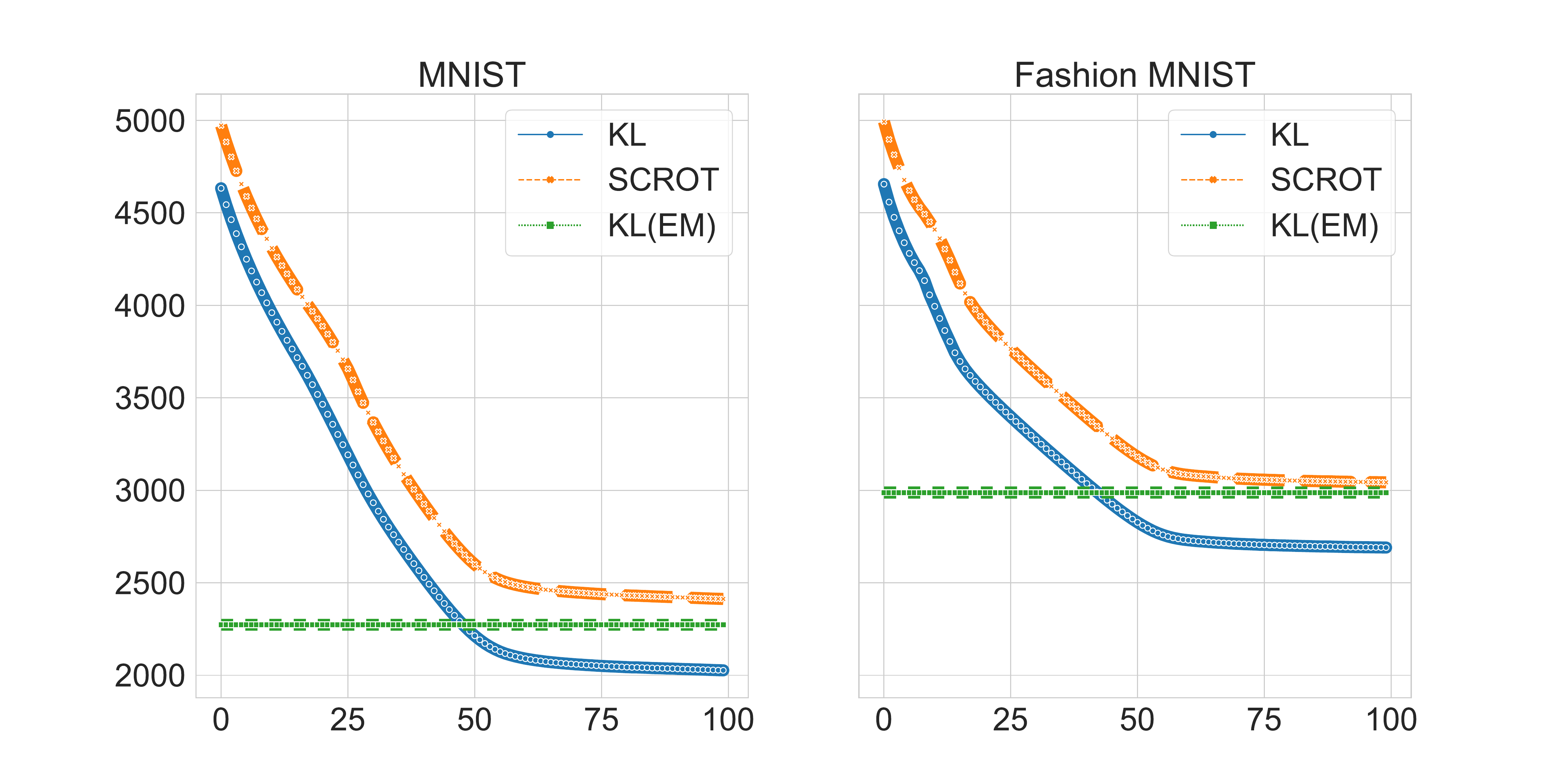}
\caption{Testing error against the number of epochs on MNIST (left) and
Fashion-MNIST (right). The curve ``KL'' shows the
estimated KLD between the data distribution (KDE based on the testing dataset)
and the learned GMM. The curve ``SCROT'' shows the SCROT distance (the learning cost function).
The curve ``KL(EM)'' shows the KLD between the data distribution
and a GMM learned using \texttt{sklearn}'s EM algorithm.\label{fig:mnist}}
\end{figure}

\section{Conclusion}

We defined the generic {\em Chain Rule Optimal Transport} (CROT) distance (Definition~\ref{def:CROT}) $H_D$ for any ground distance $D$. CROT unifies and generalizes the Wasserstein/EMD distance between discrete measures~\citealt{EMD-2000} and the {\em Mixture Component Optimal Transport}~\citep{OTWeight-2000} distance.
We proved that $H_D$ is a metric whenever $D$ is a metric (Theorem~\ref{thm:CROTmetric}).
We then dealt with statistical mixtures, and showed that $H_D(m_1,m_2)\geq D(m_1,m_2)$ (Theorem~\ref{thm:ubjcd}) whenever $D$ is jointly convex,
and considered the smooth Sinkhorn CROT distance $S_D(m_1,m_2)$ (SCROT) for fast calculations of $H_D(m_1,m_2)$ via matrix scaling algorithms (Sinkhorn-Knopp algorithm) so that
 $D(m_1,m_2)\leq H_D(m_1,m_2)\leq S_D(m_1,m_2)$.
These bounds hold in particular for statistical $f$-divergences $I_f(p:q)=\int p(x)f(q(x)/p(x)) \dx$ which includes the Kullback-Leibler divergence).
Finally, we proposed a novel efficient method to learn Gaussian mixture models from a semi-SCROT distance that bypasses Sinkhorn iterations and uses a simple normalization (Eq.~\ref{eq:normalize}). Our learning method by KDE simplification is shown to outperform the EM algorithm of {\tt sklearn} for the MNIST and Fashion MNIST datasets.

\section*{Acknowledgments}
Frank Nielsen thanks Professor Steve Huntsman for pointing out reference~\cite{OTWeight-2000} to his attention.
The authors are grateful to Professor Patrick Forr\'e (University of Amsterdam) for letting us know of an earlier error in the definition of CROT, and to Professor R\"uschendorf for sending us his work~\cite{ruschendorf}.

\bibliography{CROT}

\appendix

\section{Proof of CROT Metric (Theorem~\ref{thm:CROTmetric})}  %

\begin{proof}
We prove that $H(p,q)$ satisfies the following axioms of metric distances:

\begin{description}
\item[Non-negativity.]
As $D\bigg(p(x|y), q(x|z)\bigg)\ge0$, we have by definition that $H_D(p,q)\ge0$.

\item[Law of indiscernibles.]
If $H_D(p,q)=0$,
then $\forall\epsilon>0$,
$\exists{r}^\star\in\Gamma(p(y),q(z))$, such that
\begin{equation*}
E_{r^\star(y,z)} D\left(p(x|y),q(x|z)\right)<\epsilon.
\end{equation*}
As $D(\cdot,\cdot)$ is a metric, the density $r^\star(y,z)$ is concentrated on the region $p(x|y)=q(x|z)$
so that
\begin{equation*}
\int r^\star(y,z) p(x|y) \dy\dz
=
\int r^\star(y,z) q(x|z) \dy\dz.
\end{equation*}
We therefore have
\begin{align*}
p(x)
&=
\int p(y)p(x|y) \dy
=
\int r^\star(y,z) \dz p(x|y) \dy
=
\int r^\star(y,z) p(x|y) \dy\dz\nonumber\\
&=
\int r^\star (y,z) q(x|z) \dy\dz
=
\int r^\star (y,z) \dy q(x|z) \dz
=
\int q(z) q(x|z) \dz\nonumber\\
&=
q(x).
\end{align*}

\item[Symmetry.]
\begin{align*}
H_D(p,q) &=
\inf_{r\in\Gamma(p(y),q(z))}
\int r(y,z) D\bigg(p(x|y), q(x|z)\bigg)\,\dy\dz\nonumber\\
&=
\inf_{r\in\Gamma(p(y),q(z))}
\int r(y,z) D\bigg(q(x|z), p(x|y)\bigg)\,\dy\dz\nonumber\\
&= \inf_{R\in\Gamma(q(z),p(y))}
\int R(z,y) D\bigg(q(x|z), p(x|y)\bigg)\,\dz\dy\\
&=H_D(q,p),
\end{align*}
where $R(z,y)=r(y,z)$ s.t.
$\int R(z,y) \dy  = q(z)$ and  $\int R(z,y) \dz  = p(y)$.

\item[Triangle inequality.]
Denote
\begin{align*}
r_{12} &= \argmin_{r\in\Gamma(p_1(y_1),\;p_2(y_2))} E_{r(y_1,y_2)} D( p_1(x|y_1), p_2(x|y_2) ),\\
r_{23} &= \argmin_{r\in\Gamma(p_2(y_2),\;p_3(y_3))} E_{r(y_2,y_3)} D( p_2(x|y_2), p_3(x|y_3) ).
\end{align*}
\begin{align*}
& H_D(p_1,p_2) + H_D(p_2,p_3)\nonumber\\
= &
  E_{r_{12}(y_1,y_2)} D( p_1(x|y_1), p_2(x|y_2) )
+ E_{r_{23}(y_2,y_3)} D( p_2(x|y_2), p_3(x|y_3) )\nonumber\\
\ge &
\inf_{s} E_{s(y_1,y_2,y_3)}
\left[
 D( p_1(x|y_1), p_2(x|y_2) ) +
 D( p_2(x|y_2), p_3(x|y_3) )
 \right]\nonumber\\
\ge &
\inf_{s} E_{s(y_1,y_2,y_3)}  D( p_1(x|y_1), p_3(x|y_3) )\nonumber\\
= &
\inf_{r} E_{r(y,z)}  D( p_1(x|y), p_3(x|z) )\\
= & H_D(p_1,p_3),
\end{align*}
\end{description}
where $s(y_1,y_2,y_3)$ denotes the set of all probability measures on $\calY^3$ with marginals $p_1$, $p_2$ and $p_3$.
Clearly, $\frac{r_{12}(y_1,y_2)r_{23}(y_2,y_3)}{p_2(y_2)}\in s(y_1,y_2,y_3)$.
\end{proof}

\section{Proof of upper bound of $H_D$ }

Without loss of generality we assume $p$ and $q$ are mixture models.
The proof for the general case is similar.

\begin{proof}
\begin{eqnarray*}
D(m_1:m_2)&=& D\left(\sum_{i=1}^{k_1} \alpha_ip_i,\sum_{j=1}^{k_2} \beta_j q_j\right)\\
&=& D\left(\sum_{i=1}^{k_1} \sum_{j=1}^{k_2} w_{i,j} p_{i,j} :\sum_{i=1}^{k_1} \sum_{j=1}^{k_2} w_{i,j} q_{i,j}\right)\\
&\leq & \sum_{i=1}^{k_1} \sum_{j=1}^{k_2} w_{i,j} D(p_{i,j}:q_{i,j}),\\
&\leq & \sum_{i=1}^{k_1} \sum_{j=1}^{k_2} w_{i,j} D(p_{i}:q_{j})
 =: H_D(m_1,m_2).
\end{eqnarray*}

\end{proof}

\section{Upper bounding  $f$-divergences\label{KL:UB}}

First, let us start by proving the following lemma for the Kullback-Leibler divergence:
\begin{lemma}\label{eq:lemmaubkl}
The Kullback-Leibler divergence between two  Radon-Nikodym $p$ and $q$ with respect to $\mu$
is upper bounded as follows:
$\KL(p:q)\leq \int \frac{p(x)^2}{q(x)}\dmu(x) -1$.
\end{lemma}

\begin{proof}
Consider a strictly convex and differentiable function $F(x)$ on $(0,\infty)$.
Then we have
\begin{equation}\label{eq:ineq}
F(b)-F(a)\geq F'(a)(b-a),
\end{equation}
for any $a,b\in (0,\infty)$, with equality iff. $a=b$.
Indeed, this inequality is related to the non-negativeness of the scalar Bregman divergence $B_F(b,a)=F(b)-F(a)-(b-a)F'(a)\geq 0$.

Plugging $F(x)=-\log x$ (with $F'(x)=-\frac{1}{x}$ and $F''(x)=\frac{1}{x^2}>0$), $a=q(x)$ and $b=p(x)$ in Eq.~\ref{eq:ineq},
we get
$$
\log q(x)-\log p(x) \geq \frac{q(x)-p(x)}{q(x)}.
$$
Multiplying both  sides of the inequality by $-p(x)<0$ (and reversing the inequality), we end up with
$$
p(x)\log \frac{p(x)}{q(x)} \leq \frac{p^2(x)}{q(x)}-p(x).
$$
Then taking the integral over the support $\calX$ of the distributions yields:
$$
\KL(p:q) \leq \int_\calX \frac{p(x)^2}{q(x)}\dmu(x) -1,
$$
with equality when $p(x)=q(x)$ almost everywhere.
Notice that the right-hand side integral $\int_\calX \frac{p(x)^2}{q(x)}\dmu(x)$ may diverge (e.g., when KL is infinite).
\end{proof}

Now, let us consider two mixtures $m(x)=\sum_{i=1}^{k} w_ip_i(x)$ and $m'(x)=\sum_{i=1}^{k'} w_i'p_i'(x)$.
Apply Lemma~\ref{eq:lemmaubkl} to get
$$
\KL(m:m') \leq \sum_{i,j}  \int w_iw_j\frac{p_i(x)p_j(x)}{m'(x)} \dmu(x)-1.
$$

Let us upper bound $A_{ij}=\int \frac{p_i(x)p_j(x)}{m'(x)}\dmu(x)$ to upper bound
$$
\KL(m:m') \leq \sum_{i,j}  w_iw_jA_{ij}-1.
$$

For bounding the terms  $A_{ij}$, we interpret the mixture density as an arithmetic weighted mean that is greater or equal than a geometric mean (AGM inequality).
Therefore we get:
$$
\int \frac{p_i(x)p_j(x)}{m'(x)}\dmu(x)  \leq \int \frac{p_i(x)p_j(x)}{\prod_{l=1}^{k'} w_l'p_l'(x)}\dmu(x).
$$

When the mixture components belong to a same exponential family~\cite{EF-2009}, we get a closed-form upper bound since $\theta_i+\theta_j-\sum_{l=1}^{k'} w_l'\theta_l'\in\Theta$:
Let $\bar\theta'=\sum_{l=1}^{k'} w_l'\theta_l'$ denote the barycenter of the  natural parameters of the mixture components of $m'$.
We have:
$$
\frac{p(x;\theta_i)p(x;\theta_j)}{\prod_{l=1}^{k'} w_l'p(x;\theta_l')}=
\exp\left(\left(\theta_i+\theta_j-\bar\theta'\right)^\top t(x)-F(\theta_i)-F(\theta_j)+\sum_{l=1}^{k'} w_l' F(\theta_l')+ k(x)\right).
$$

Taking the integral over the support we find that
$$
A_{ij} \leq \exp\left(F\left(\theta_i+\theta_j-\bar\theta'\right)-F(\theta_i)-F(\theta_j)+\sum_{l=1}^{k'} w_l' F(\theta_l')\right).
$$

Overall, we get the upper bound:
\begin{equation}
\KL(m:m') \leq \left(\sum_{i,j}  w_iw_j\exp\left(F\left(\theta_i+\theta_j-\bar\theta'\right)-F(\theta_i)-F(\theta_j)+\sum_{l=1}^{k'} w_l' F(\theta_l')\right) \right) -1.
\end{equation}

In general, we have the following upper bound for $f$-divergences~\cite{Dragomir-2000}:

\begin{property}[$f$-divergence upper bound]
The $f$-divergence between two densities $p$ and $q$ with respect to $\mu$ is upper bounded as follows:
 $I_f(p:q) \leq \int (q(x)-p(x))f'\left(\frac{q(x)}{p(x)}\right)\dmu(x)$.
\end{property}

\begin{proof}
Let us use the non-negative property of scalar Bregman divergences:
$$
B_F(a:b) = F(a)-F(b)-(a-b)F'(b) \geq 0.
$$
Let $F(x)=f(x)$ (with $F(1)=f(1)=0$), and $a=1$ and $b=\frac{q}{p}$.
It follows that
$$
B_F\left(1:\frac{q}{p}\right)= -f\left(\frac{q}{p}\right)-\left(1-\frac{q}{p}\right)f'\left(\frac{q}{p}\right) \geq 0.
$$
That is,
$$
pf\left(\frac{q}{p}\right) \leq p\left(\frac{q}{p}-1\right)f'\left(\frac{q}{p}\right).
$$
Taking the integral over the support, we get
$$
I_f(p:q) \leq \int (q-p)f'\left(\frac{q}{p}\right)\dmu.
$$
\end{proof}

For example, when $f(u)=-\log u$ (with $f'(u)=-\frac{1}{u}$), we recover the former upper bound:
$$
\KL(p:q) \leq \int (p-q)\frac{p}{q}\dmu = \int \frac{p^2}{q}\dmu -1.
$$
Notice that $\int \frac{p^2}{q}\dmu-1$ is  a $f$-divergence for the generator $f(u)=\frac{1}{u}-1$.

\section{Square root of the symmetric $\alpha$-Jensen-Shannon divergence}

TV is bounded in $[0,1]$ which makes it difficult to appreciate the quality of the CROT upper bounds in general.
We shall consider a different parametric distance $D_\alpha$ that is upper bounded by an arbitrary bound: $D_\alpha(p,q)\leq C_\alpha$.

It is well known that the square root of the Jensen-Shannon divergence is a metric~\citep{JS-2004} satisfying the triangle inequality.
In~\cite{symmetricJ-2010}, a generalization of the Jensen-Shannon divergence was proposed, given by
\begin{equation}
\JS_\alpha(p:q) \eqdef \frac{1}{2}\KL(p:(pq)_\alpha)+\frac{1}{2}\KL(q:(pq)_\alpha),
\end{equation}
where $(pq)_\alpha\eqdef (1-\alpha)p+\alpha q$.
$\JS_\alpha$ unifies (twice) the Jensen-Shannon divergence (obtained when $\alpha=\frac{1}{2}$) with the Jeffreys divergence ($\alpha=1$;~\citealt{symmetricJ-2010}).
A nice property is that the skew $K$-divergence is upper bounded as follows:
$$
\KL(p:(pq)_\alpha)\leq \int p\log \frac{p}{(1-\alpha)p} \leq -\log(1-\alpha)
$$
for $\alpha\in (0,1)$, so that $\JS_\alpha[p:q] \leq -\frac{1}{2}\log(1-\alpha)-\frac{1}{2}\log\alpha$ for $\alpha\in (0,1)$.

Thus, we have the square root of the symmetrized $\alpha$-divergence that is upper bounded by
$$
\sqrt{\JS_\alpha(p:q)} \leq  C_\alpha =\sqrt{-\frac{1}{2}\log(1-\alpha)-\frac{1}{2}\log\alpha}.
$$

However, $\sqrt{\JS_\alpha[p:q]}$ is not a metric in general~\citep{Osterreicher-2003}.
Indeed, in the extreme case of $\alpha=1$, it is known that any positive power of the Jeffreys divergence does not yield a metric.

Observe that $\JS_\alpha$ is a $f$-divergence since $K_\alpha(p:q)\eqdef\KL(p:(pq)_\alpha)$ is a $f$-divergence for the generator
$f(u)=-\log((1-\alpha)+\alpha u)$, and we have $\KL(q:(pq)_\alpha)=K_{1-\alpha}(q:p)$. Since $I_f(q:p)=I_{f^\diamond}(p:q)$ for $g(u)=uf(1/u)$, it follows that the $f$-generator $f_{\JS_\alpha}$ for the $\JS_\alpha$ divergence is:

\begin{equation}
f_{\JS_\alpha}(u)= -\log\left((1-\alpha)+\alpha u\right) - \log\left (\alpha+\frac{1-\alpha}{u}\right).
\end{equation}

Figure~\ref{fig:jscrot} and table~\ref{tab:jsmnist} display the experimental results obtained for the $\alpha$-JS divergences.
One can have similar observations with the TV results.

\def\ttt{1.0}

\begin{figure}
\centering
\includegraphics[width=\ttt\columnwidth]{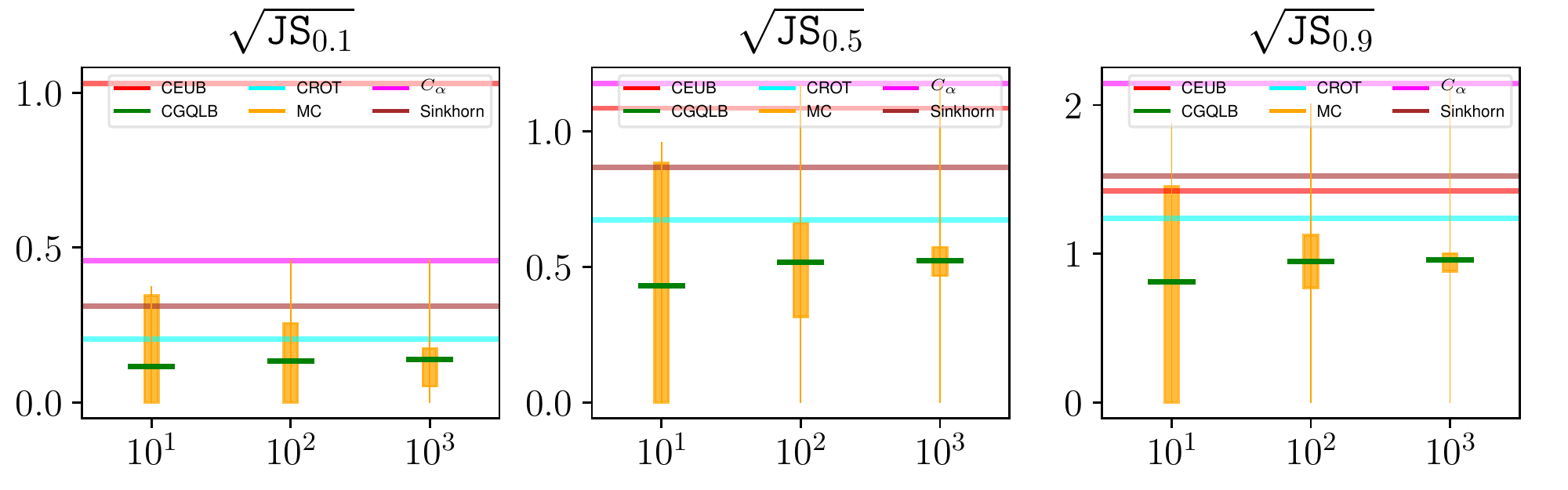}\\
\includegraphics[width=\ttt\columnwidth]{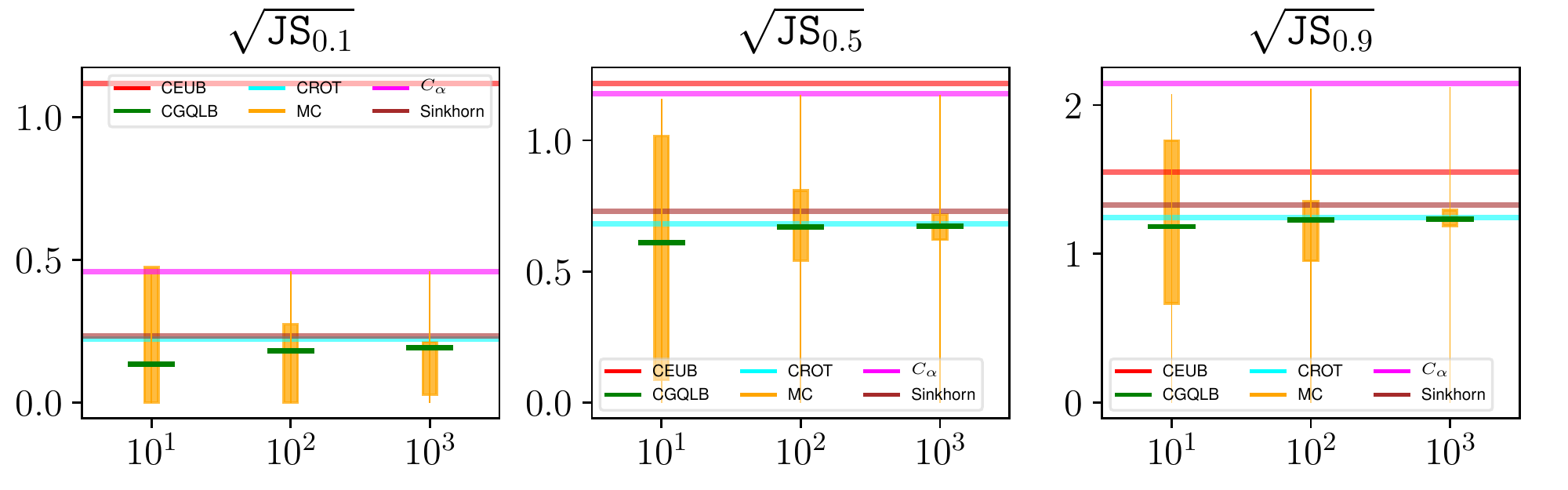}\\
\includegraphics[width=\ttt\columnwidth]{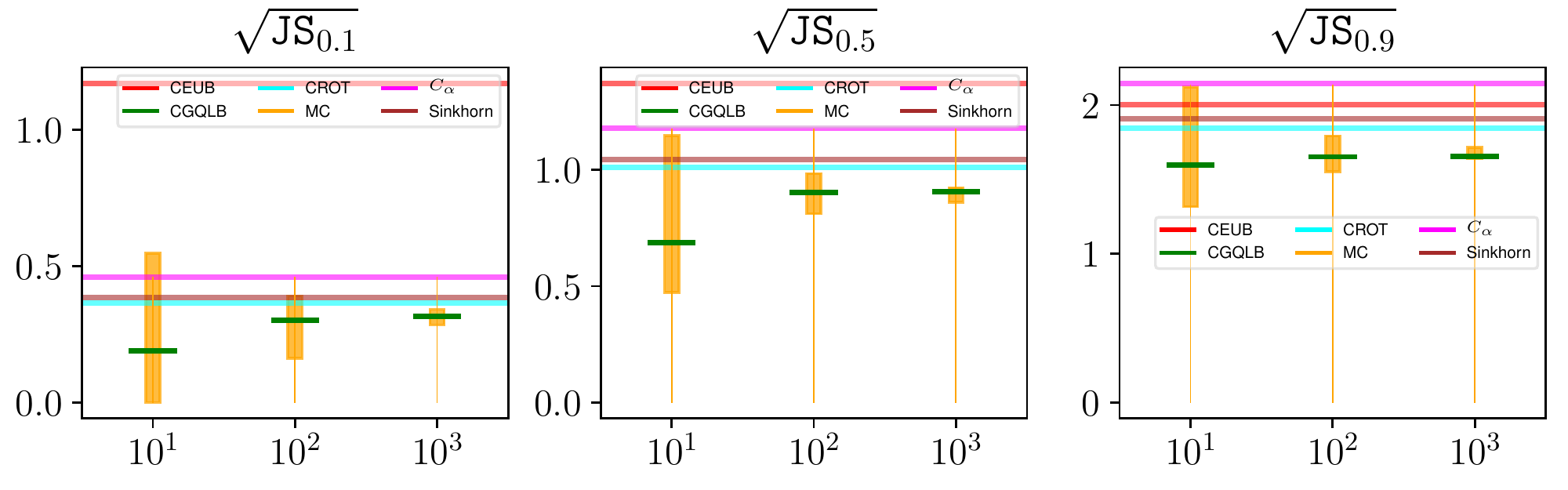}\\
\caption{Performance of the CROT distance and the Sinkhorn CROT distance for upper bounding the square root of the $\alpha$-Jensen-Shannon distance between mixtures of (1) Gaussian, (2) Gamma, and (3) Rayleigh distributions.}%
\label{fig:jscrot}%
\end{figure}

\begin{table}
\caption{Square root of the Jensen-Shannon divergence between two 10-component GMMs
estimated on PCA-processed images.\label{tab:jsmnist}}
\centering
\begin{tabular}{ccc|ccccc}
\hline
Data & $D$ & $\tau$ & $\sqrt{\mathrm{JS_{0.5}}}$ & CROT-$\sqrt{\mathrm{JS}_{0.5}}$ & Sinkhorn ($10$) & Sinkhorn ($1$)\\
\hline
\multirow{4}{*}{MNIST}
& $10$ & $1$ & $0.25\pm0.11$ & $0.36\pm0.17$ & $0.37\pm0.17$ & $0.94\pm0.05$\\
& $10$ & $0.1$ & $0.39\pm0.05$ & $0.55\pm0.07$ & $0.56\pm0.08$ & $1.00\pm0.02$\\
& $50$ & $1$ & $0.51\pm0.11$ & $0.54\pm0.12$ & $0.56\pm0.13$ & $0.93\pm0.04$\\
& $50$ & $0.1$ & $0.69\pm0.05$ & $0.76\pm0.07$ & $0.79\pm0.07$ & $1.00\pm0.03$\\
\hline
& $10$ & $1$ & $0.33\pm0.15$ & $0.31\pm0.13$ & $0.33\pm0.14$ & $0.96\pm0.04$\\
Fashion
& $10$ & $0.1$ & $0.46\pm0.09$ & $0.48\pm0.09$ & $0.49\pm0.10$ & $1.01\pm0.03$\\
MNIST
& $50$ & $1$ & $0.60\pm0.12$ & $0.57\pm0.14$ & $0.59\pm0.15$ & $1.03\pm0.04$\\
& $50$ & $0.1$ & $0.75\pm0.07$ & $0.76\pm0.09$ & $0.80\pm0.10$ & $1.08\pm0.02$\\
\hline
\end{tabular}
\end{table}

\section{Visualization of the optimal transport assignment problem of CROT and MCOT distances}

Figure~\ref{fig:CROT} illustrates the principle of the CROT distance.

\begin{figure}%
\centering
\includegraphics[width=0.8\columnwidth]{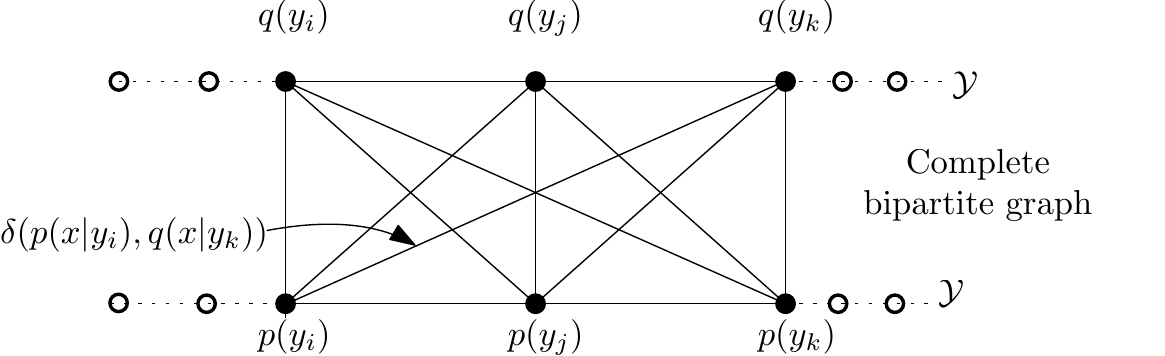}%
\caption{The CROT distance: Optimal matching of marginal densities \wrt a distance on conditional densities.
We consider the complete bipartite graph with edges weighted by the distances
$D$ between the corresponding conditional densities defined at edge vertices.}%
\label{fig:CROT}%
\end{figure}

\begin{figure}%
\centering
\includegraphics[width=0.85\columnwidth]{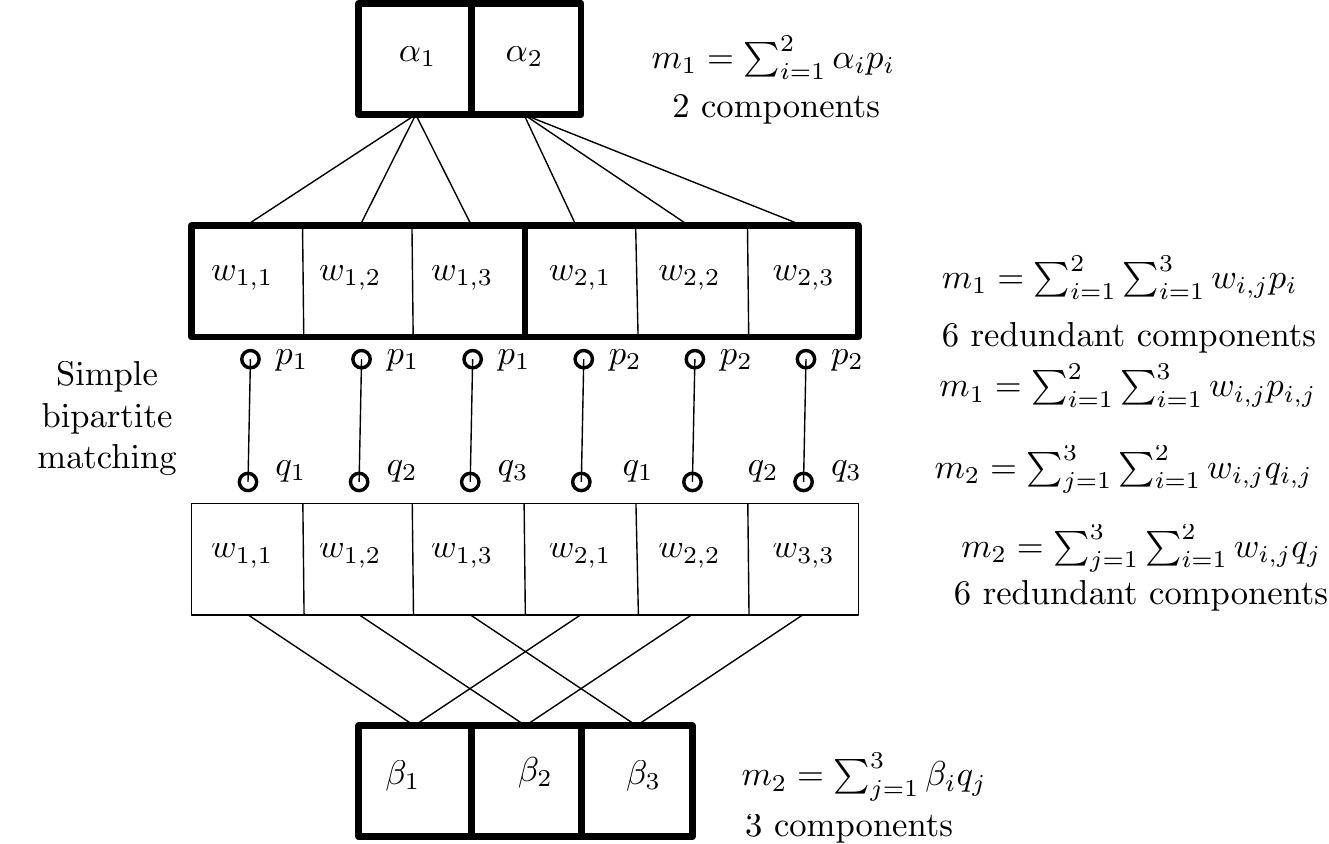}%

\caption{An interpretation of CROT by rewriting the mixtures
$m_1=\sum_{i=1}^{k_1} \sum_{j=1}^{k_2} w_{i,j} p_{i,j}$ and $m_2=\sum_{i=1}^{k_1} \sum_{j=1}^{k_2} w_{i,j} q_{i,j}$ with $p_{i,j}=p_i$ and $q_{i,j}=q_j$ and using the joint convexity of the base distance $D$. }%
\label{fig:MCOT}%
\end{figure}

\end{document}